\documentclass[onecolumn,10pt]{asme2e}

\usepackage{fancyhdr} 
\usepackage{multirow}
\usepackage{array}
\usepackage{url}
\usepackage{graphicx,subfigure}
\usepackage{stackrel}
\usepackage{lipsum}
\usepackage{amssymb}
\usepackage{epstopdf}
\usepackage{epsfig} 
\usepackage{color}
\usepackage[linesnumbered,ruled]{algorithm2e}
\usepackage{calligra}
\usepackage[T1]{fontenc}
\usepackage[noend]{algorithmic}
\usepackage{chngcntr}
\usepackage{mathtools}
\usepackage{mathbbol}

\usepackage{chngcntr}
\counterwithout{figure}{section}
\counterwithout{table}{section}
\counterwithout{table}{subsection}
\counterwithout{figure}{subsection}

\newtheorem{theorem}{Theorem}

\newtheorem{lem}{Lemma}

\newdef{example}{Example}
\newdef{definition}{Definition}
\newdef{dfn}{Definition}
\newdef{rmk}{Remark}
\newdef{remark}{Remark}
\newdef{prob}{Problem}
\newdef{problem}{Problem}
\newdef{cond}{Condition}
\title{Correct-by-Construction Approach for Self-Evolvable Robots}

\author{Gang Chen
    \affiliation{
    Department of Mechanical and Aerospace Engineering\\
	University of California, Davis\\
	One Shields Avenue,Davis, CA\\
    ggchen@ucdavis.edu
    }	
}

\author{Zhaodan Kong    
    \affiliation{
    Department of Mechanical and Aerospace Engineering\\
	University of California, Davis\\
	One Shields Avenue,Davis, CA\\
	zdkong@ucdavis.edu
    }
}

\begin{document}

\maketitle    

\begin{abstract}
The paper presents a new formal way of modeling and designing reconfigurable robots, in which case the robots are allowed to reconfigure not only structurally but also functionally. We call such kind of robots ``self-evolvable'', which have the potential to be more flexible to be used in a wider range of tasks, in a wider range of environments, and with a wider range of users. To accommodate such a concept, i.e., allowing a self-evovable robot to be configured and reconfigured, we present a series of formal constructs, e.g., structural reconfigurable grammar and functional reconfigurable grammar. Furthermore, we present a correct-by-construction strategy, which, given the description of a workspace, the formula specifying a task, and a set of available modules, is capable of constructing during the design phase a robot that is guaranteed to perform the task satisfactorily. We use a planar multi-link manipulator as an example throughout the paper to demonstrate the proposed modeling and designing procedures. 
\end{abstract}



\section{Introduction}

Reconfigurable robots are a family of robots that are capable of adjusting their shapes and functions to changing environments and tasks \cite{yim2007modular,ahmadzadeh2016modular}. They are posed to meet the increasing demands of providing personal robots to adjust to individual needs and physical characteristics \cite{harada2009wireless,satici2009design} as well as industrial robots to adapt to changes in the market \cite{farid2014axiomatic}. Over the past three decades, the field of reconfigurable robots has advanced from proofs-of-concept to physical implementations. However, even with their potential versatility and robustness over conventional robots, reconfigurable robots still suffers from inferior performance, one of the main factors impeding them from practical adoption. Furthermore, existing reconfigurable robots are rarely capable of functional adaption. In this paper, we propose a formal modeling framework of reconfigurable robots that are capable of both structural and functional reconfigurations. We will also explore a design philosophy called ``correct-by-construction'' to guarantee the performance of the robots during the design phase.  

Formally the approaches of studying reconfigurable robots can be roughly divided into three categories, those based on graph theory, those based on optimization, and those based on dynamic analysis. Graph-theory-based approaches are mostly suitable to study how modules are put together structurally \cite{yim2007modular,eckenstein2015modular,neubert2016soldercubes,sung2015reconfiguration}. Modules are represented as vertices while connections between the modules are represented as edges. Then tools from graph theory can be used to solve problems related to reconfigurable robots, such as configuration recognition \cite{park2008automatic} and motion planning \cite{sung2015reconfiguration}. Optimization-based approaches cast the design of a reconfigurable robot as an optimization problem with a objective function over the vector of design variables \cite{freitas2010kinematic,ferguson2007flexible}. The design variables, either discrete or continuous, are subjected to equality and/or inequality constraints. The optimization-based approaches are suitable to address trade-offs among multiple competing objectives. The detailed kinematics/dynamics of the designed robots are generally either ignored or simplified in the first two types of approaches, while the last type of approaches, dynamic-analysis-based, puts kinematics/dynamics as the main focus \cite{balmaceda2016novel,cucu2015towards}. Currently papers employing dynamic-analysis-based approaches mostly deal with arm robots \cite{djuric2010global,balmaceda2016novel} with some exceptions dealing with mobile robots \cite{cucu2015towards}. One issue with the aforementioned approaches is their inability to allow for simultaneously structural and functional reconfigurations, thus greatly restrict the potential of their robots. In this paper, we will develop a formal framework incorporating both types of reconfigurations. We will call such type of robots as self-evolvable robots. Notice that in existing literature, reconfiguration generally refers to structural changes, i.e., units/modules change the way they connect to each other mechanically. In this paper, we will adopt a rather broader definition of reconfiguration to include functional changes within each unit (we will focus on using changes in dynamics due to some physical parameters, e.g., the length of a link, as an example of functional changes in this paper). This is inspired by natural evolution, i.e., a biological mechanism (analogous to a robot) gradually changes its shapes and functions to adapt to changes in the environment (analogous to changes in missions). For the rest of the paper, we will use self-evolvable robots and reconfigurable robots interchangeably.


This paper is organized along the line of modeling and design as follows. Section 2 discusses the modeling of self-evolvable robots. Section 3 formally defines the design problem. Section 4 presents the method to solve the design problem. Section 5 provides a case study to demonstrate our method. Section 6 concludes the paper.

\section{Modeling of Self-Evolvable Reconfigurable Robots}

In this section, we will first describe a list of modules that will be used in this paper to construct self-evolvable robots. The list is not meant to be exhaustive but mainly serves as a running example for the rest of the paper. Next, we will introduce two definitions, structural reconfiguration grammar (SRG) and structural reconfiguration automaton (SRA), which formally characterize the way the modules are mechanically/structurally connected to each other to form a robot. Then we will introduce dynamic models of modules. Finally, we will introduce two additional definitions, functional reconfiguration grammar (FRG) and functional reconfiguration automaton (FRA), which formally characterize the way to (re)configure a robot not only structurally but also functionally. 

\subsection{Modules}
\label{sec:modules}

Reconfigurable robots have the capacity to deliberately change their own structures by adaptively rearranging the connectivity of their components  according to the environments and/or task scenarios \cite{yim2007modular}. The repeatable building components of a reconfigurable robot are called modules or mechanical units. They usually have uniform docking interfaces allowing different modules to connect to each other mechanically and electronically. 


\begin{figure}[thb!]
\centering
\subfigure[]{\includegraphics[width=0.6\textwidth]{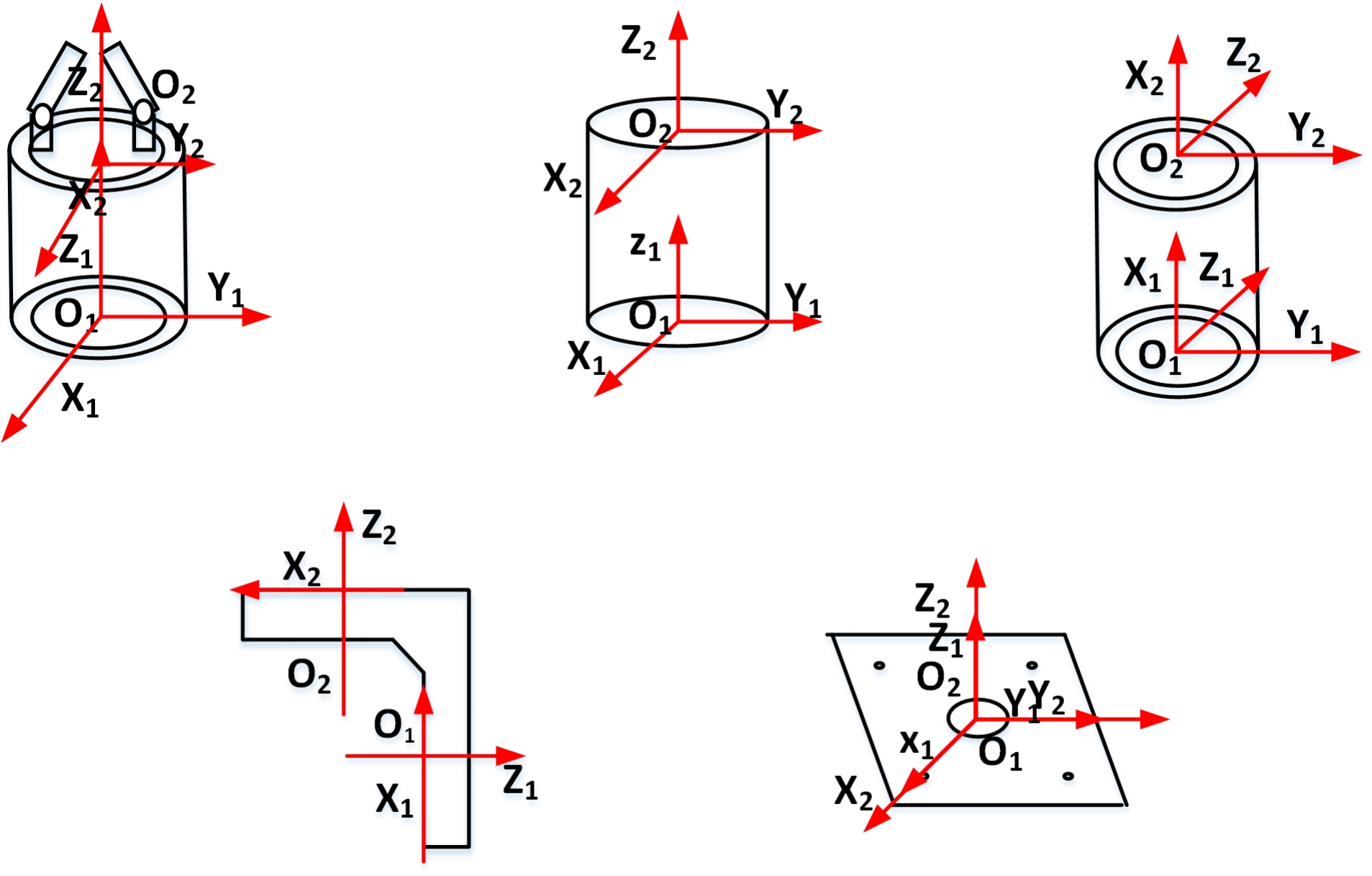} \label{fig:modules_1}}
\subfigure[]{\includegraphics[width=0.6\textwidth]{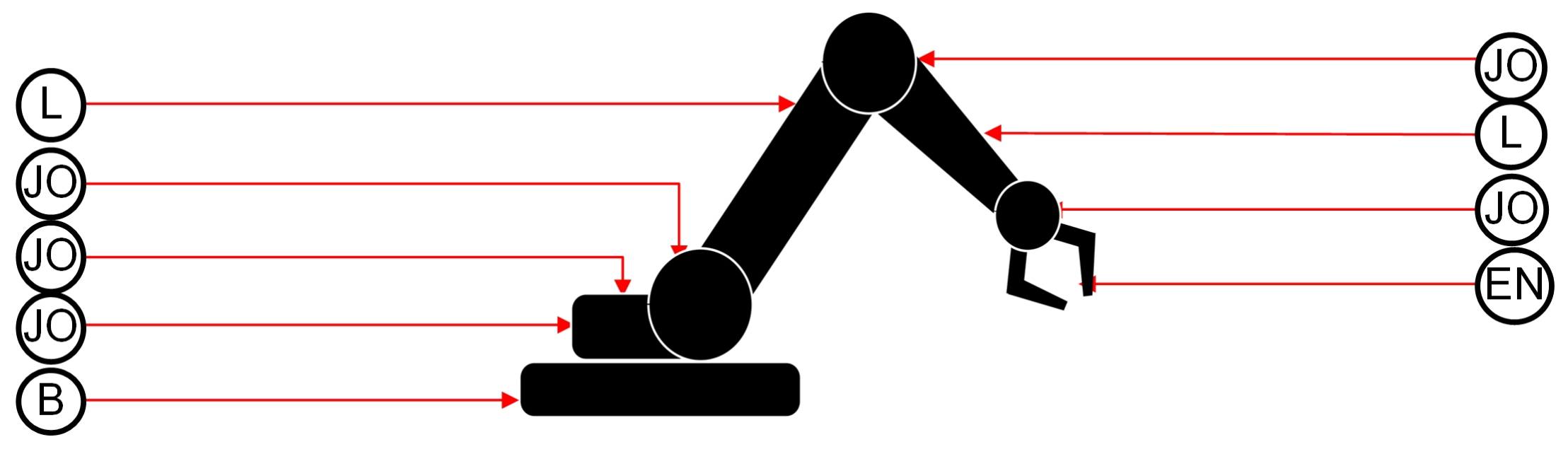} \label{fig:modules_2}}
\subfigure[]{\includegraphics[width=0.6\textwidth]{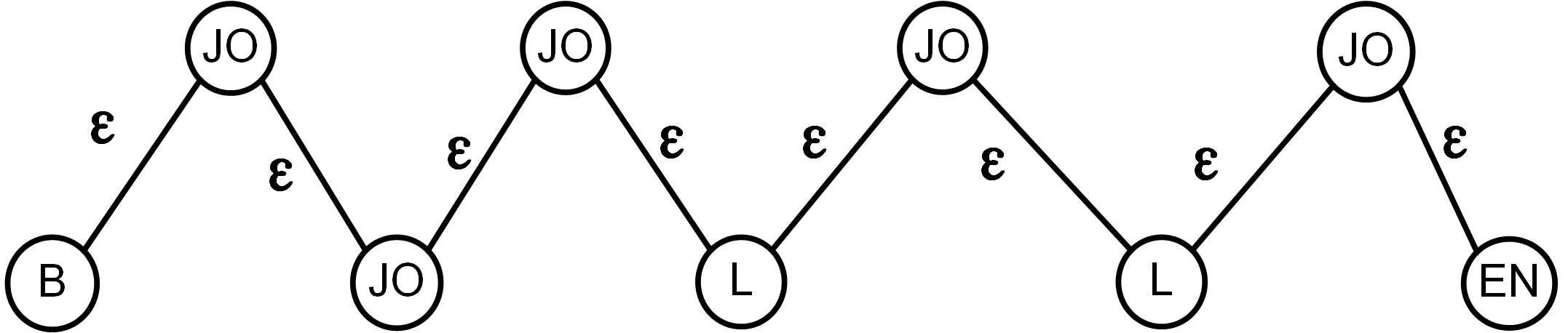} \label{fig:modules_3}}
\caption{(a) The modules that will be used in this paper to construct reconfigurable robots. They are (starting from the upper left corner in the clockwise direction) end-effector module, joint module, cylindrical link module, base module, and L-shaped link module. (b) An illustrative example of a robot built from the modules. (c) The $(\Sigma,\Gamma)$ labeled graph representation of the robot shown in (b). The sequence of symbols representing the configuration is $B \epsilon JO \epsilon JO \epsilon  JO \epsilon  L \epsilon  JO \epsilon  L \epsilon  JO \epsilon  EN$.}
\end{figure}

In the following text, we will describe four types of modules as shown in Fig. \ref{fig:modules_1}. Each module has an input end denoted by subscript 1 and an output end denoted by subscript 2. Information comes into the module via the input end and gets out of the module via the output end. Two coordinate frames are attached to the two ends of the module for the purpose of characterizing the dynamics of various parts of the robot. To illustrate the functional reconfigurability, some module is associated with a design parameter, which can be adjusted thus changing the functionality of the module. 

\textit{Joint Module:} As shown by the upper middle sub-figure of Fig. \ref{fig:modules_1}, the joint module $F$ is modeled as a cylinder with an axis of rotation $O_{1}O_{2}$. An input frame $JO_{1}$ is attached to the input connector/end at point $O_{1}$ and an output frame $JO_{2}$ is attached to the output connector/end at point $O_2$. The z-axes of the two frames both coincide with the line $O_{1}O_{2}$ while their x and y axes define the two end planes.  

\textit{Link Module:} In this paper, we define two different types of link modules, the cylindrical link module as shown at the upper right corner of Fig. \ref{fig:modules_1} and the L-shaped link module at the lower left corner of Fig. \ref{fig:modules_1}. The cylindrical link module is modeled similar to the link module. It is a cylinder with an axis of rotation $O_{1}O_{2}$. An input frame $L_1$ and an output frame $L_2$ are attached to the two ends at point $O_1$ and point $O_2$, respectively, with their Z-axes having the same direction as $O_{1}O_{2}$ and their X- and Y-axes defining the two end planes. The L-shaped link, on the other hand, has its input end and out end perpendicular to each other. An input frame $L_1$ is attached to the input end of the module with its z-axis perpendicular to the input end plane and its X- and Y-axes defining the input end plane. An output frame $L_2$ is defined similarly with respect to the output end. Each link module has a designing parameter $p_L$, which is the length of the module. 

\textit{End-Effector Module:} The end-effector module is the functional component of the robot. It has a variety of forms, i.e., a mechanical gripper and a machine tool base. In this paper, as shown at the upper left corner of Fig. \ref{fig:modules_1}, we will use a gripper as an example of the end-effector module. An input frame $EN_1$ is attached to its input end with its Z-axis perpendicular to the end plane and its X- and Y-axes defining the end plane. An output frame $EN_2$ is defined in such a way that its origin is at the grasping center of the fingers.

\textit{Base Module:} The base module serves as the base for other modules. As shown at the lower right corner of Fig. \ref{fig:modules_1}, an input frame $B_1$ is attached to the input end of the module, which is attached to the ground, while an output frame $B_2$ is attached to the output end of the module in such a way that the origin of the frame $B_2$ is located at the center of the base, its Z-axis is perpendicular to the output end plane, and its X- and Y-axes define the output end plane.

A module can be connected to another one as long as the input framework of one module coincide with the output framework of the other. Of course, in order to build a functional robot, some further requirements need to be taken into consideration, e.g., the base module must be attached to the ground and there must be at least one end-effector module. An example of such a robot built from the modules is illustrated in Fig. \ref{fig:modules_2}.


\subsection{Structural (Re)Configuration}



In this subsection, we will present two definitions, structural reconfiguration grammar (SRG) and structural reconfiguration automaton (SRA). Both of them can characterize   the way to reconfigure a robot structurally/mechanically.  

\subsubsection{Structural Reconfiguration Grammar (SRG)}

Let's first define $(\Sigma,\Gamma)$ labeled graph, modified from a definition called $\Sigma$-labeled $\Gamma$-graph in \cite{reiter2015distributed}.
\begin{dfn}
($(\Sigma,\Gamma)$ labeled graph) \cite{reiter2015distributed}: Let $\Sigma$ and $\Gamma$ be two finite nonempty sets of node labels and edge labels, respectively. Let $G=(V,E)$ be a directed graph, where $V$ is the set of nodes and $E$ is the set of directed edges. The graph $G$ can be labeled by a function $l: E \rightarrow (\Sigma, \Gamma)$ with the node labeling $l_V: V \rightarrow \Sigma$ and the edge labeling $l_E: E \rightarrow \Gamma$. The tuple  $\langle G, \Sigma, \Gamma\rangle$ is called a $(\Sigma,\Gamma)$ labeled graph (or simply labeled graph) and denoted by $G_{(\Sigma,\Gamma)}$.
\end{dfn}

Next, we will define the structural reconfiguration grammar (SRG).

\begin{dfn}
\label{dfn:RGG}
\textit{(Structural Reconfiguration Grammar, SRG)}: A reconfiguration graph grammar $SRG$ is a tuple $SRG=(\Sigma, \Gamma, N, P, I)$, where $\Sigma$ is a finite alphabet of node symbols or tokens, $\Gamma$ is a finite alphabet of edge symbols or tokens, $N$ is a finite set of symbols called non-terminals, $P$ is a finite set of mappings $N \rightarrow (\Sigma \cup \Gamma \cup N)^*$ called production rules with superscript $(\cdot)^*$ as a notation for the set of all strings over an alphabet $(\cdot)$, and $I \in \Sigma$ is the initial node symbol.
\end{dfn}

The production rules can be conveniently written in Backus-Naur form \cite{aho1979introduction}, $N\rightarrow X_1 X_2 ... X_n$, where $N$ is some non-terminal and $X_1 X_2 ... X_n$ is a sequence of node/edge symbols and non-terminals. A production rule indicates that $N$ may expand to all strings represented by the right hand side of the rule. The collection of all sequences of \textit{terminal} symbols/tokens, i.e., those in $\Sigma$ or $\Gamma$, generated by the SRG is called the language of the SRG, denoted by $L(SRG) \subset (\Sigma \cup \Gamma)^*$.

\begin{example}
\label{example:1}
\textit{(SRG for a reconfigurable, planar, multi-link manipulator robot)} Given the four types of modules described in Sec. \ref{sec:modules}, $JO$, the joint module, $L$, the link module, $EN$, the end-effector module, and $B$, the base module, a SRG for a reconfigurable, planar, multi-link robot is $SRG=(\Sigma, \Gamma, N, P, I)$ with (1) $\Sigma = \{JO, L, EN, B\}$, the collection of modules; (2) $\Gamma = \{\epsilon \}$ meaning that there is no restriction on the way one module is connected to another one; (3) $N$ is the collection of $(\Sigma,\Gamma)$ labeled graphs with each element corresponding to a structural configuration of the robot; (4) $P: N \rightarrow B | N \epsilon JO | N \epsilon L | N \epsilon EN$ characterizing how the robot is configured; and (5) $I = B$. An illustration is shown in Fig. \ref{fig:modules_3}.
\end{example}


The structural reconfiguration grammar (SRG) can be defined alternatively as follows:
\begin{dfn}
\label{dfn:RGG2}
\textit{(Structural Reconfiguration Grammar, SRG, Alternative Definition)}: A reconfiguration graph grammar $SRG$ is a tuple $SRG=(Z, N, P, I)$, where $Z =\{\Sigma,\Gamma\}$ is a finite alphabet of symbols or tokens, $P$ is a finite set of mappings $N \rightarrow (Z \cup N)^*$ called production rules, and the others have the same meanings as in Definition \ref{dfn:RGG}.
\end{dfn}

\subsubsection{Structural Reconfiguration Automaton (SRA)}

Context-free grammars (CFGs), such as those in Definition \ref{dfn:RGG} and Definition \ref{dfn:RGG2}, have equivalent representations as pushdown automata (PDA) which recognize the language of the grammar \cite{aho1979introduction}. A pushdown automaton is a automaton with a stack, which provides the automaton with memory. The automaton corresponding to SRG, called structural reconfiguration automata (SRA), can be defined as follows:
 
\begin{dfn}
\label{dfn:RGA}
\textit{(Structural Reconfiguration Automaton, SRA)}: A reconfiguration graph automaton $SRA$ is a tuple $SRA=(Q, Z, \delta, Q_0, A)$, where $Q$ is a finite set of states, $Z$ is a finite alphabet of symbols/tokens, $\delta: Q \times Z \rightarrow Q$ is the transition function, $Q_0$ is the initial state, $A \in Q$ is the set of accept states.
\end{dfn}

Let $SRA=(Q, Z, \delta, Q_0, A)$ be a SRA and $\omega=z_1 ... z_n \in Z^*$ a finite word. A run for $\omega$ in $SRA$ is a finite sequence of states $q_0 q_1 ... q_n$ such that: (i) $q_0 \in Q_0$; (ii) $q_i \rightarrow_{z_{i+1}} q_{i+1}$ for all $0 \leq i < n$ where $\rightarrow_{.}$ is defined by the transition $\delta$ as $q \rightarrow_{z} q'$ if and only if $q' \in \delta (q, z)$. Runs $q_0 q_1 ... q_n$ is called accepting if $q_n \in A$. A finite word $\omega \in Z^*$ is called accepted by $SRA$ is there exists an accepting run for $\omega$. The accepted language of $SRA$, denoted by $L(SRA)$, is the set of finite words in $Z^*$ accepted by $SRA$, i.e., $L(SRA)=\{\omega \in Z^* | \text{ there exists an accepting run for} \omega \text{ in } SRA\}$.

The idea behind the construction of a PDA from a CFG is to have the PDA simulate the sequence of left- or right-sentential forms that the grammar uses to generate a given terminal string $\omega$ \cite{aho1979introduction,baier2008principles,dantam2013motion}. For a reconfiguration graph grammar $SRG$, there is a unique equivalent reconfiguration graph automaton $SRA$ such that $L(SRG)=L(SRA)$. For a given SRG $SRG=(Z, N, P, I)$, its equivalent SRA $SRA=(Q, Z, \delta, Q_0, A)$ is constructed as follows: (1) $Q=(N \cup Z)^*$; (2) they share the same $Z$; (3) $q' \in \delta (q, z)$ if $q \rightarrow q'=qz$ is a production rule with $q \in N$, $z \in Z$, and $q' \in (N \cup Z)^*$; (4) $Q_0 = I$; and (5) $A=Z^*$. Even though SRA and SRG are equivalent, they can be used for different purposes. For instance, SRG, given its constructive form, is more intuitive, while SRA, given its automaton format, is easier to be integrated with other formal verification and synthesis techniques, such as model checking \cite{baier2008principles}.  

\begin{figure}[thb!]

\centering
\includegraphics[width=0.45\textwidth]{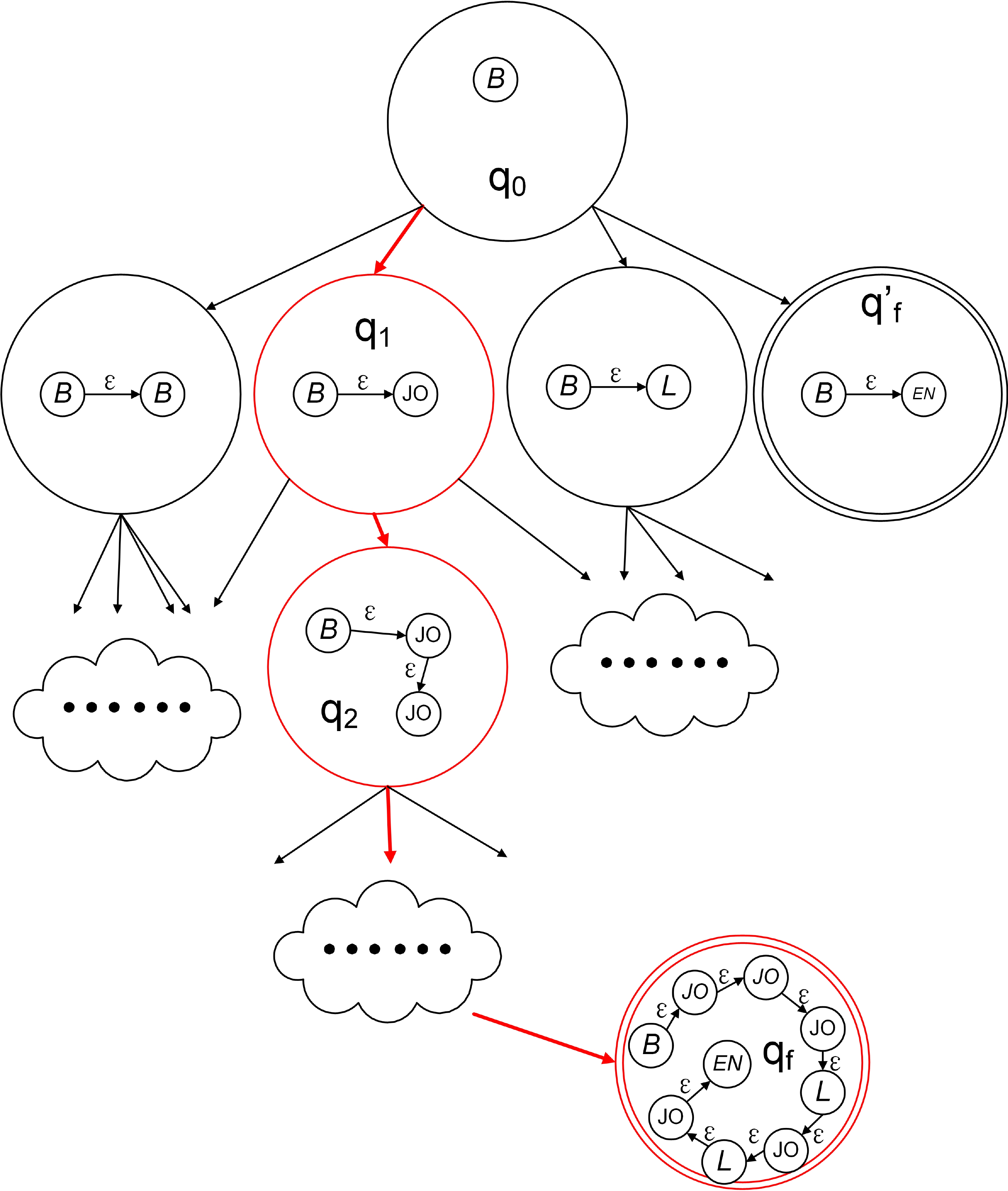}
\caption{Part of the reconfiguration graph automaton (SRA) of the reconfigurable robot illustrated in Example \ref{example:1} and Fig. \ref{fig:modules_3}. The initial state $q_0$ is indicated by having an incoming arrow without source. The accept states are indicated by double circles. The red nodes and edges indicate the accepted run to construct the robot in Fig. \ref{fig:modules_3}.}
\label{fig:RGA}
\end{figure}

\begin{example}
\textit{(SRA of the reconfigurable robot in Example \ref{example:1})} Part of the SRA of the reconfigurable robot illustrated in Example \ref{example:1} and Fig. \ref{fig:modules_1} is shown in Fig. \ref{fig:RGA}. A state of the SRA corresponds to a labeled graph representation of a structural configuration with only one initial state $q_0 = B$, i.e., the the base module. A transition between two states represents the addition or removal of a module. For instance, the transition from $q_0$ to $q_1$ represents to connection of a joint module $JO$ to the base module $B$. To specify the requirement that a functional robot must have an end-effector, we restrict the set of accept states to $A=\{\omega \in Z^*| \exists i, \text{ such that } \omega_i = EN\}$, i.e., at least one of the modules need to be an end-effector module. An example of such accept states is $q_f$ in Fig. \ref{fig:RGA}. The steps to construct a robot structurally can then be represented as an accepted run of the corresponding SRA as shown in the figure.
\label{example_2}
\end{example}

\subsection{Models of Modules}

Each module of a reconfigurable robot has a unique function, which depends on some continuous or discrete design parameters, e.g., the dimension of the module. For instance, the length of a link module of a robot determines the configuration space of the robot, i.e., whether a certain position and orientation can be achieved.  Here we introduce a definition of modules as one of the modeling bases to allow a robot to reconfigure not only its structure but only its functions, i.e., the set of design parameters.  

\begin{dfn}
\label{dfn:modules}
\textit{(Model of Modules)}: The function of a modules $\sigma \in \Sigma$ is defined as a parametric controlled dynamical system $F_{\sigma}=(X_{\sigma}, \Xi_{\sigma}, U_{\sigma}, f_{\sigma})$, where $X_{\sigma} \subset \mathbb{R}^n$ is the state space, $\Xi_{\sigma} \subset \mathbb{R}^p$ is the parameter space, $U_{\sigma} \subset \mathbb{R}^m$ is the control space, $f_{\sigma}:X_{\sigma} \times \Xi_{\sigma} \times U_{\sigma} \rightarrow X_{\sigma}$ is an analytic vector field, assumed to be sufficiently smooth, and $n$, $p$, $m$ are the dimensions of $X_{\sigma}$, $\Xi_{\sigma}$ and $U_{\sigma}$, respectively.
\end{dfn} 

Notice that the modules are controlled not autonomous, meaning that a designer or the robot itself has the freedom to specify a control policy $u_{\sigma} \in U_{\sigma}$ for a module $\sigma$.

\subsection{Functional (Re)Configuration}



The models of modules, combined with the concept of SRG, lead naturally to the following concept called Functional Reconfiguration Grammar (FRG):
\begin{dfn}
\label{RG}
\textit{(Functional Reconfiguration Grammar, FRG)}: A functional reconfiguration grammar $FRG$ is a tuple $FRG=(N,Z,P,F,I)$, where $N$, $Z$, $P$ and $I$ are defined the same as in Definition \ref{dfn:RGG2} and $F:=\{F_{\sigma}, \sigma \in \Sigma\}$ with each $F_{\sigma}$ defined the same as in Definition \ref{dfn:modules}.
\end{dfn}

Given a set of modules with their dynamics described by parametric dynamical systems (Definition \ref{dfn:modules}) and as a grammar, structural reconfiguration grammar (SRG) in our case, describing how these modules can be structurally connected, the above definition gives rise to a range of dynamics that can emerge from the whole robots. Such dynamics can be considered as the results of the semantic interpretation of the syntax of the functional reconfiguration grammar (FRG), i.e., given a production rule in the corresponding SRG, a semantic rule can be generated by a parser; applying a sequence of production rules in SRG gives the structural configuration of the robot, while applying the corresponding sequence of semantic rules of the corresponding FRG gives the functional configuration (the dynamics in our case) of the robot. The aforementioned points can be best understood with an example. 

\begin{example}
\label{example_3}
\textit{(FRG of the reconfigurable robot in Example \ref{example_2})} The linearized dynamic model of a reconfigurable robot constructed in Example \ref{example_2} can be described as follows: 
\begin{equation}
\label{modelofrobot}
[M]\{\ddot{x}\}+[C]\{\dot{x}\}+[K]\{x\}=\{T\}-\{F\}
\end{equation}
where 
\begin{equation*}
\begin{array}{lll}
\{x\}=\{\triangle q\}^{T};\quad
\left[ M \right]=[E_{2}]^{T}[E_{2}]\\
\left[ C\right] =[E_{2}]^{T}[E_{1}]+[\dot{E}_{2}]^{T}[E_{2}]+[E_{2}]^{T}[\dot{E}_{2}]-[E_{1}]^{T}[E_{2}]\\
\left[ K\right] =[\dot{E}_{2}]^{T}[E_{1}]+[E_{2}]^{T}[\dot{E}_{1}]-[E_{1}]^{T}[E_{1}]\\
\{T\}=\{\triangle \tau\}^{T}\\
\{F\}= [\dot{E}_{2}]^{T}[E_{0}]+[E_{2}]^{T}[\dot{E}_{0}]-[E_{1}]^{T}[E_{0}]\\
\end{array}
\end{equation*}
with $\triangle q$ as the vector of perturbed link poses and $\triangle \tau$ as the vector of perturbed torques \cite{chen2001dynamic}. 

Here a set of semantic rules can be introduced to construct the matrix $E$ for the set of production rules in structural reconfiguration grammar (SRG). Specifically, if the production rule is $N\rightarrow N\varepsilon JO$, i.e., the newly added module is a joint, then $E$ will be kept the same; if the production rule is $N\rightarrow N\varepsilon L$, i.e., the newly added module is a link; let's call the new link as the $n$-th link and index existing ones as 1-st link, 2-nd link and so on, according to the order they are added, then $E$ will be updated as follows:
\begin{equation*}
\label{dynamicchange}
\begin{array}{lll}
\left[ E_{i0} \right] =
\left\lbrace 
\begin{matrix}
\sum_{l=1}^{i}-L_{l}\dot{q}_{0l}\sin q_{0l}\\
\sum_{l=1}^{i}L_{l}\dot{q}_{0l}\cos q_{0l}\\
\end{matrix}
\right\rbrace \\

\left[ E_{i1} \right] =

\begin{bmatrix}
-L_{1}\dot{q}_{01}\cos q_{01}&-L_{2}\dot{q}_{02}\cos q_{02}&\cdots&-L_{i}\dot{q}_{0i}\cos q_{0i}\\
-L_{1}\dot{q}_{01}\sin q_{01}&-L_{2}\dot{q}_{02}\sin q_{02}&\cdots&-L_{i}\dot{q}_{0i}\sin q_{0i}\\
\end{bmatrix}
 \\

\left[ E_{i2} \right] =
\begin{bmatrix}
-L_{1}\sin q_{01}&L_{2}\cos q_{02}&\cdots&-L_{i}\sin q_{0i}\\
L_{1}\cos q_{01}&L_{2}\cos q_{02}&\cdots&L_{i}\cos q_{0i}\\
\end{bmatrix}
 \\
\\
\left( \left[ E_{k} \right]^{T}\left[ E_{l} \right]\right)_{n}=\left( \left[ E_{k} \right]^{T}\left[ E_{l} \right]\right)_{n-1}+L_{n}[E_{nk}]^{T}[E_{nl}]
\end{array}
\end{equation*}
where $L_i$ is the length of the i-th link with $i=1,...,n$ and $k,l=0,1,2$, $k\geq l$.
\end{example}

An automaton, called functional reconfiguration automaton (FRA), that is equivalent to a functional reconfiguration grammar (FRG), can be constructed similar to the way that a structural reconfiguration automaton (SRA) is constructed from a structural reconfiguration grammar (SRG). We are going to omit the definition here to save space. 

%

\section{Design Problem Statement}


The robot's workspace can be represented by a set of polytopes $P=\{P_i, i=1,...,p\}$. Each polytope $P_i$ is assigned with an atomic proposition $\pi_i \in \Pi=\{\pi_t,\pi_o,\pi_f\}$, where $\pi_t$, $\pi_o$ and $\pi_f$ stand for ``target region'', ``obstacle region'' and ``free region'', respectively. The adjacency relationship among the polytopes can be encoded by an adjacency matrix $N=[N_{i,j},i,j=1,...,p]$ with $N_{i,j}$ as one if polytope $i$ and polytope $j$ are neighboring regions, zero otherwise. Finally, there is a projection function $\mathcal{H}: X \rightarrow \Pi$ which maps a robot's state to its corresponding atomic proposition.

\begin{prob}
\label{problem_1}
Given a functional reconfiguration grammar $FRG=(N,Z,P,F,I)$ and a workspace description $\mathcal{W}=(P,\mathcal{H},N)$, find a finite sequence of symbols $\omega=z_1...z_n \in Z^*$ and a finite sequence of parameters $\theta_{\omega}=\theta_{z_1}...\theta_{z_n}$ with $\theta{z_i} \in \Theta_{z_i}$ such that: (i) $\omega$ is accepted by the corresponding $SRG$; (ii) there exists a trajectory $x_0,...,x_k$ of the robot built in accordance with $\omega$ and $\theta_{\omega}$, satisfying the following formula $\phi$:
\begin{equation}
\label{specification}
\mathcal{H}(x_k)=\pi_t \wedge_{i=0}^{k-1} \mathcal{H}(x_i) = \pi_f  \wedge_{i=0}^{k-1} N(\mathcal{H}(x_i),\mathcal{H}(x_{i+1}))=1.  
\end{equation}  
\end{prob}

\begin{rmk}
The above formula essentially specifies a motion planning (pick-and-place) problem, i.e., the constructed robot should be able to move from a starting region to the target region while in the meantime avoiding all obstacles. Such a way of specifying the problem may seem awkward. There are two reasons for such a choice: one is to enable us to use off-the-shelf solvers to find a feasible path, and the other one is to keep the option open for future extensions. For instance, we are interested in using richer logic specifications, such as linear temporal logic \cite{baier2008principles} and signal temporal logic \cite{kong2016temporal,aksaray2016q}, in the future.   
\end{rmk}

\begin{rmk}
Solving the problem requires solving the following two sub-problems. The first one is a structural synthesis problem. We need find an $\omega$ such that it is accepted by the corresponding $SRG$, meaning that we need to build a robot that is structurally feasible, e.g., it must start from a base and end up with an effector. The solution of this sub-problem is a robot with its structural configuration fixed, i.e., the set of selected modules and the way they are connected to each other are determined. The second sub-problem is a functional synthesis problems, involving selecting a parameter $\theta_{z_i}$ for each module $F_{z_i}$ in such a way that a feasible trajectory can be generated by the constructed robot. Notice that since each module is modeled as a parametric controlled dynamical system, even after the parameters have been chosen for all the modules, we still need to to check whether there exists a control policy to solve the problem. The first sub-problem is an easy one, given the formulation of the definition of SRG or SRA. So next we are going to focus on solving the second sub-problem.   
\end{rmk}


\section{Functional Synthesis}

Before embarking upon presenting the solution to the functional synthesis problem, let's first introduce a concept called configuration robustness.   


\subsection{Configuration Robustness}

Once the structural (encoded by $\omega$, see Problem \ref{problem_1}) and functional (encoded by $\theta_{\omega}$, see Problem \ref{problem_1}) configuration of a robot has been determined, the dynamics of the robot will be determined as well, as demonstrated by the Example \ref{example_3}. The equation describing such dynamics, e.g., Eqn. (\ref{modelofrobot}), can be written in its state space form as follows:
\begin{equation}
\label{state_space}
x_{i+1} = A(x_{i}) x_{i} + B(x_{i}) u_i.
\end{equation}

\begin{dfn}
\textit{(Configuration Robustness)}: Given a configuration $(\omega, \theta_{\omega})$, a workspace description $\mathcal{W}=(P,\mathcal{H},N)$, a finite trajectory of the corresponding robot $\bar{x}=x_0,...,x_k$, and a formula $\phi$, e.g., Eqn. (\ref{specification}), the configuration robustness $\rho$ is defined as follows:  
\begin{equation}
\label{dfn:robustness}
\rho(\omega, \theta_{\omega}, \mathcal{W}, \bar{x}, \phi)=\max_{
\begin{array}{ll}
u_0,\cdots,u_{k-1}\in \mathbb{R}^m\\
v_0,\cdots,v_{k-1}\in \mathbb{R}^m
\end{array}}(-\max_{
\begin{array}{ll}
s_{0}^{u},\cdots,s_{k-1}^{u}\in \mathbb{R}\\
s_{0}^{v},\cdots,s_{k-1}^{v}\in \mathbb{R}
\end{array}}(s_{0}^{u}+s_{0}^{v},s_{1}^{u}+s_{1}^{v},\cdots,s_{k-1}^{u}+s_{k-1}^{v}))
\end{equation}
subject to
\begin{equation*}
\begin{array}{lll}
(C.1) \quad \mathcal{H}(\bar{x}) \models \phi;\\
(C.2) \quad x_{i+1}=A(x_{i}) x_{i}+B(x_{i}) u_{i}+B^{'}v_{i},\quad i=0,\cdots,k-1;\\
(C.3) \quad\parallel u_{i}\parallel \leq \overline{u}+s_{i}^{u},\quad i=0,\cdots,k-1;\\
(C.4) \quad\parallel v_{i}\parallel \leq s_{i}^{v},\quad i=0,\cdots,k-1;\\
(C.5) \quad 0 \leq s_{i}^{v},\quad i=0,\cdots,k-1;\\
(C.6) \quad \epsilon \left( \sum\limits_{l=0}^{i-1}s_{l}^{u}+s_{l}^{v} \right)\leq s_{i}^{u}+s_{i}^{v},\quad i=1,\cdots,k-1.\\
\end{array}
\end{equation*}
$\textit{C.1}$ says that the trajectory $\bar{x}$ must satisfy the specification $\phi$. $\textit{C.2}$ says that $\bar{x}$ is a feasible trajectory of the robot. $B^{'}$ is a matrix to make $[B(x_{i}),B^{'}]$ surjective. $v_i$ is an additional control input. $\textit{C.3}$, $\textit{C.4}$, and $\textit{C.5}$ constrain the input $u_i$ and the additional input $v_i$ by slack variables $s^u$ and $s^v$. These slack variables are added to relax the dynamics constraints. $\overline{u}$ is a bound on the magnitude of the control input. Finally, $\textit{CR.6}$ provides a user specified bound $\epsilon$ on the slack variables.
\label{configurationrobustness}
\end{dfn}

\begin{theorem}
\label{node}
Given two structural configurations $\omega_1$ and $\omega_2$ with $\omega_1=\omega_2 z$, i.e., $\omega_2$ is a prefix of $\omega_1$, then the following relationship holds:
\begin{equation*}
\rho(\omega_1, \theta_{\omega_1}^*, \mathcal{W}, \bar{x}, \phi) \geq \rho(\omega_2, \theta_{\omega_2}^*, \mathcal{W}, \bar{x}, \phi) 
\end{equation*}
where $\theta_{\omega_1}^*$ and $\theta_{\omega_2}^*$ are the optimal parameters for the two structural configurations $\omega_1$ and $\omega_2$, respectively, in term of configuration robustness.
\end{theorem}
\begin{proof}
The proof of this theorem can be found in the Appendix.
\end{proof}

\begin{theorem}
\label{robustness}
Given a functional reconfiguration grammar $FRG=(N,Z,P,F,I)$ and a workspace description $\mathcal{W}=(P,\mathcal{H},N)$, there is a solution to Problem (\ref{problem_1}) if and only if there exists an $\omega$, a $\theta_{\omega}$, and a trajectory $\bar{x}$ generated by the robot built in accordance with $\omega$ and $\theta_{\omega}$, such that  
\begin{equation*}
\rho(\omega, \theta_{\omega}, \mathcal{W}, \bar{x}, \phi) \geq 0
\end{equation*}
\end{theorem}

\begin{proof}
The proof of this theorem can be found in the Appendix.
\end{proof}

\begin{algorithm} \label{CorrectbyConstruction}
 \caption{Correct-by-Construction for Self-Evolvable Reconfigurable Robots}
 \KwIn{\\ Workspace description $\mathcal{W}=(P,\mathcal{H},N)$, a functional reconfiguration grammar $FRG$, an initial configuration $q_0$ of the FRG}
\begin{algorithmic}[1]
\STATE $\omega = q_0$
  \STATE $\mathcal{W}^*=WS.Abstraction(\mathcal{W})$;
  \STATE \textbf{while} No feasible trajectory has been found \textbf{do}
  \STATE \quad $\kappa=SAT(P^*,\mathcal{H}^*,N^*)$ ($P.Planning$);
  \STATE \quad  $\theta_{\omega}^* = argmax [\rho(\omega, \theta_{\omega}, \mathcal{W}^*, \bar{x}, \kappa)]$ ($P.Synthesis$);
  \STATE \qquad  \textbf{if} $\rho(\omega, \theta_{\omega}^*, \mathcal{W}, \bar{x}, \kappa)<0$ \textbf{then}
  \STATE \qquad    $\phi_{c}=CounterExample$;
  \STATE \qquad    $\kappa=\kappa\wedge\phi_{c}$;
   \STATE \qquad   \quad \textbf{if} $\kappa:=\emptyset$ \textbf{then}
   \STATE \qquad\qquad   $\omega = S.Synthesis(\omega)$;
  \STATE \qquad  \textbf{else}
   \STATE \qquad   \qquad \textbf{break};
  \STATE \textbf{return} Configuration $\omega, \theta_{\omega}$.
\end{algorithmic}
\end{algorithm}



Our proposed algorithm to solve Problem \ref{problem_1} is briefly outlined in Algorithm 1. It involves solving three main sub-problems. The first problem is a path planning problem ($P.Planning$): given certain abstraction of the workspace, it finds a path $\kappa=(\kappa_0,...,\kappa_k)$ satisfying
\begin{equation}
\label{path_specification}
\kappa_0 = \bar{\kappa} \wedge \kappa_k = \pi_t \wedge_{i=1}^{k-1} \kappa_i = \pi_f
\end{equation}
where $\bar{\kappa}$ is the starting region, $\pi_t$ is the target region, and $\pi_f$ is a free region. Essentially the problem entails finding a path from the starting region to the target region while avoiding all obstacles. The second problem is a parameter synthesis problem ($P.Synthesis$): given a path $\kappa$ and the current structural configuration $\omega$ of the robot, it finds an optimal parameter $\theta_{\omega}$ to optimize the configuration robustness (refer to Theorem \ref{robustness} for the rationale). Finally, we need to solve a structural synthesis problem ($S.Synthesis$), i.e., to find the next structural configuration to be considered. The last problem is an easy one as mentioned. So we are going to focus on solving the other two problems.
 
\subsection{Path Planning}

Before solving the path planning problem, we first abstract the description of the workspace $\mathcal{W}=(P,\mathcal{H},N)$ further by using the technique described in \cite{ding2011automatic}. The corresponding function in Algorithm (1) is $WS.Abstraction$. The end result is a coarse abstraction of the workspace. Each region $\mathcal{W}_i$ is a polytope, mathematically specified by a set of linear constraints $C_{\mathcal{W}_{i}}x \leq b_{\mathcal{W}_{i}}$. Notice that such an abstraction is a refinement of the original description of the workspace. Thus the original proposition of a region should be inherited, i.e., if $\mathcal{W}_{j}$ is a refinement of $P_i$, then $\pi_{\mathcal{W}_{j}} = \pi_{P_i}$. 

Notice that propositions attached to the regions are atomic. Furthermore, the path planning specification, Eqn. (\ref{path_specification}), is written in propositional logic. Thus given the abstraction of the workspace, $\mathcal{W}^*$, off-the-shelf SAT solvers can be used to efficiently solve the path planning problem \cite{shoukryscalable}. The corresponding function in Algorithm 1 is $SAT$. In the future, we are planning to replace the specification with richer ones, such as those written in linear temporal logic \cite{baier2008principles}. In that case, SMT solvers are needed \cite{komuravelli2014smt}. 

\subsection{Parameter Synthesis}

According to Theorem 2, the parameter synthesis problem ($P.Synthesis$ in Algorithm 1) entails to an optimization problem, i.e., given the current structural configuration $\omega$ of the robot, a workspace description $\mathcal{W}=(P,\mathcal{H},N)$ (after the abstraction), and a path $\kappa=(\kappa_0,...,\kappa_k)$ computed by the path planning algorithm ($P.Planning$), find a parameter $\theta_{\omega}$ as well as its corresponding control policy $u_0^{k-1}$ (subjected to the constraint that $|u_i|\leq \bar{u}, i=0,...,k-1$) such that the configuration robustness $\rho(\omega, \theta_{\omega}, \mathcal{W}^*, \bar{x}, \kappa)$ is maximized. 

Notice that once a parameter has been selected, the configuration and subsequently the dynamics of the robot will be determined. Provided with different parameters, the corresponding control policies, if they exist, will be different. Thus the parameter synthesis requires solving two problems iteratively: 
\begin{itemize}
\item (i) Given a parameter $\theta$, a workspace description $\mathcal{W}=(P,\mathcal{H},N)$, and a path $\kappa=(\kappa_0,...,\kappa_k)$, find whether there exists a control policy $u_0^{k-1}$ for the robot to track the path without colliding with the obstacles. 
\item (ii) Find the next parameter to optimize $\rho(\omega, \theta_{\omega}, \mathcal{W}^*, \bar{x}, \kappa)$. 
\end{itemize}

For the first problem, since, in this case, the structural configuration $\omega$ is fixed and the functional configuration (described by the parameter $\theta_{\omega}$) is fixed as well. According to the semantics of functional reconfiguration grammar (FRG), the two configurations together will give rise to the dynamics of the robot, which can be linearized (see Eqn. (\ref{state_space}) for a simple example). Moreover, the workspace is described by a set of linear inequalities, $C_{\mathcal{W}_{i}}x \leq b_{\mathcal{W}_{i}}$. In summary, the first problem is linear and can be efficiently solved by linear programming algorithms. 

The second problem is more challenging and interesting. There is no closed form solution for $\rho(\omega, \theta_{\omega}, \mathcal{W}^*, \bar{x}, \kappa)$ even for simple configurations. The only way any information can be obtained regarding a particular parameter $\theta_{\omega}$ is to first of all solve the control synthesis problem (the aforementioned problem (i)) and then find out its corresponding robustness. Essentially we are trying to solve a global optimization problem with an unknown objective function $\rho$. Such kind of problems can be solved by using particle swarm optimization \cite{haghighi2015spatel}, Nelder-Mead \cite{jin2013mining}, simulated annealing \cite{kong2014temporal}, and stochastic gradient descent algorithm \cite{kong2016temporal}, etc. But it is worth pointing out that many of these techniques may suffer from slow convergence.  

To facilitate the convergence rate of the optimization, we use an active learning algorithm called Gaussian Process Adaptive Confidence Bound (GP-ACB) developed by our group in \cite{chen2016active}:
\begin{equation}
\label{gpacb}
\theta_t=argmax_{\theta \in \Theta} m_{t-1}(\theta)+\eta_m(\theta)^{\frac{1}{2}}\beta_t^{\frac{1}{2}}\sigma_{t-1}(\theta),
\end{equation}
where $t$ is the current step; $\Theta$ is the search space; $\beta_t$ is a function of $t$ and independent of $\theta$; $m_{t-1}(.)$ and $\sigma_{t-1}(.)$ are the mean and covariance functions of a Gaussian process, which is unknown and characterize the underlying configuration robustness function $\rho$, respectively; $\theta_t$ is the instance that will be inquired at step $t$, meaning the label of $\theta_t$ will be obtained from the oracle (in our case, the first problem, i.e., the control synthesis problem, will be solved and the corresponding configuration robustness will be returned); $\eta_m(\theta)$ normalizes the mean $m_{t-1}(\theta)$ and can be written explicitly as
\begin{equation*}
\eta_m(\theta)=\frac{m_{t-1}(\theta)-\min(m_{t-1}(\theta))}{\max (m_{t-1}(\theta))-\min (m_{t-1}(\theta))}. 
\end{equation*}
In the algorithm, $\eta_m(\theta)$ acts as an adaptive factor to uncertainty (covariance) and favors exploration directions associated with increasing rewards. We have shown in \cite{chen2016active} theoretically that GP-ACB outperforms many state-of-the-art active learning algorithms with similar settings, e.g., GP-UCB. We have also shown empirically that GP-ACB outperforms many state-of-the-art sampling based optimization algorithms, e.g. Nelder-Mead, by an average of 30 to 40 percent faster.   

If the optimal configuration robustness for the current structural configuration is negative, meaning that there is no feasible solution regardless of the functional configuration and the control policy, we will relax the current configuration by adding another module (remember that we have shown in Theorem 1 that doing this will always improve the optimal configuration robustness). Moreover, the found counter-example $\phi_c$ will be added to the current path specification $\kappa$ to prone the search space for the path planning algorithm.


\section{Case Study}

\begin{figure}
\centering
\subfigure[Workspace]{\label{fig:a}\includegraphics[width=0.7\textwidth]{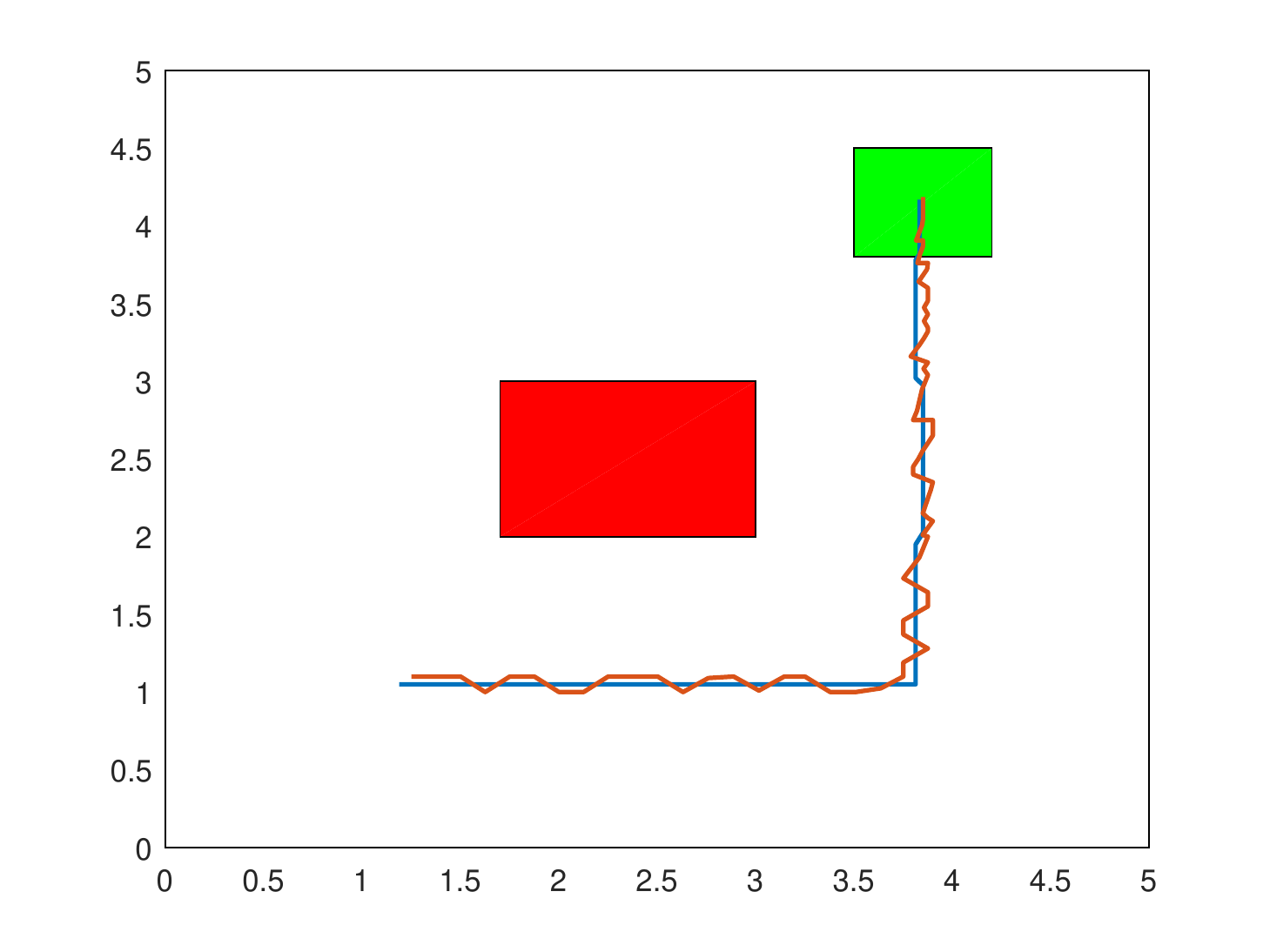}}
\subfigure[Control strategy]{\label{fig:b}\includegraphics[width=0.7\textwidth]{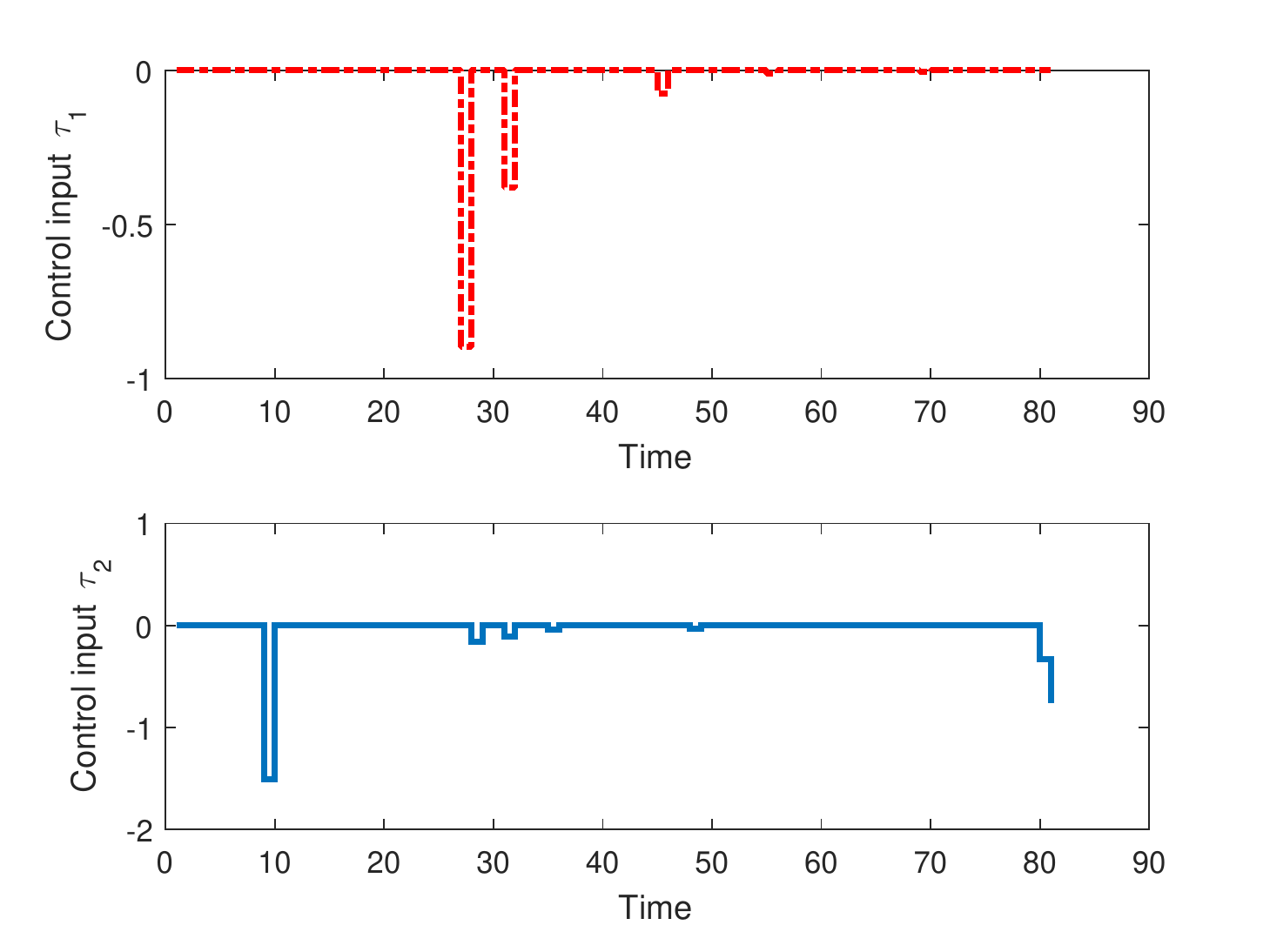}}
\caption{(a) The workspace. The obstacle region is shown in red while the target region is shown in green. The blue line is the path obtained by using the SAT solver while the red line is the actual robot trajectory. (b) The associated control strategy for two links, the first one of length 2.23 and the second one of length 3.35.}
\label{control}
\end{figure}

The following case study is based on the functional reconfiguration grammar $FRG$ constructed in the first three examples and a workspace as shown in Fig. \ref{fig:a}. In the workspace, there is an obstacle region shown in red and a target region shown in green. The two regions can be mathematically described as follows:
\begin{equation*}
\begin{bmatrix}
-1 &0\\
1&0\\
0&-1\\
0&1
\end{bmatrix}
\begin{bmatrix}
x_{o}\\
y_{o}
\end{bmatrix}
\leq 
\begin{bmatrix}
-1.7\\
3\\
-2\\
3\\
\end{bmatrix};
\quad
\begin{bmatrix}
-1 &0\\
1&0\\
0&-1\\
0&1
\end{bmatrix}
\begin{bmatrix}
x_{t}\\
y_{t}
\end{bmatrix}
\leq 
\begin{bmatrix}
-3.5\\
4.2\\
-3.8\\
4.5\\
\end{bmatrix}
\end{equation*}
We further set $\bar{u}$, the bound on the magnitude of the control input, to 10.

Essentially, we are given a set of modules and a workspace; we need to construct a robot by (i) selecting and connecting modules (structural configuration); and (ii) selecting appropriate parameters (in our case, the lengths of a link) for each module (functional configuration), such that the constructed robot is able to steer from the initial region to the target region while avoiding the obstacle. 

Once the workspace has been abstracted ($WS.Abstraction$ in Algorithm 1), the SAT-solver ($P.Planning$ in Algorithm 1) is applied to find a path $\kappa$ as shown by the blue line in Fig. \ref{fig:a}. The path $\kappa$ consists of a sequence 81 rectangular regions, i.e., $k=81$. Associated with each region is a set of linear inequalities.


It is quite obvious that there is no way for a robot with only one link to move from the starting region to the target region without hitting the obstacle. We are able to confirm this observation by using our algorithm. Basically, the linear program algorithm and the active learning algorithm, GP-ACB, are combined to solve the parameter synthesis problem ($C.Synthesis$ in Algorithm 1) and we are unable to get a solution, i.e., a parameter resulting positive configuration robustness.

Thus, the structural configuration of the robot is relaxed. According to the production rules, the structural reconfiguration automaton (SRA) will transit to the next state, which corresponds a robot with two links. According the semantics of the corresponding functional reconfiguration grammar (FRG), the dynamics of the two-linked manipulator is as follows:
\begin{equation}
\frac{d}{dt}
\begin{bmatrix}
\vartriangle\theta_1\\
\vartriangle\theta_{2}\\
\vartriangle\dot{\theta}_1\\
\vartriangle\dot{\theta}_2\\
\end{bmatrix}
=
\left[ 
\begin{array}{cccc}
0&0&1&0\\
0&0&0&1\\
\multicolumn{2}{c}{\multirow{2}{*}{-E/D}} &0&0\\
\multicolumn{2}{c}{}&0&0\\
\end{array}
\right] 
\begin{bmatrix}
\vartriangle\theta_1\\
\vartriangle\theta_{2}\\
\vartriangle\dot{\theta}_1\\
\vartriangle\dot{\theta}_2\\
\end{bmatrix}
+
\left[ 
\begin{array}{cccc}
0&0\\
0&0\\
\multicolumn{2}{c}{\multirow{2}{*}{1/D}}\\
\multicolumn{2}{c}{}\\
\end{array}
\right] 
\begin{bmatrix}
\vartriangle\tau_{1}\\
\vartriangle\tau_{2}\\
\end{bmatrix}
\label{twolinkdynamic}
\end{equation}
with 
\begin{equation*}
D=
\begin{bmatrix}
\frac{l_{1}^{3}+l_{2}^{3}+3l_{1}^{2}l_{2}}{3}+l_{2}^{2}l_{1}cos(\theta_{2})&\frac{l_{1}^{3}}{3}+\frac{l_{2}^{2}l_{1}}{2}cos(\theta_{2})\\
\frac{l_{1}^{3}}{3}+\frac{l_{2}^{2}l_{1}}{2}cos(\theta_{2})&\frac{l_{2}^{3}}{3}
\end{bmatrix}
\end{equation*}
\begin{equation*}
E=
\begin{bmatrix}
0& 0\\
0&0\\
\end{bmatrix}.
\end{equation*}

\begin{figure}[hbtp]
\centering
\includegraphics[width=0.7\textwidth]{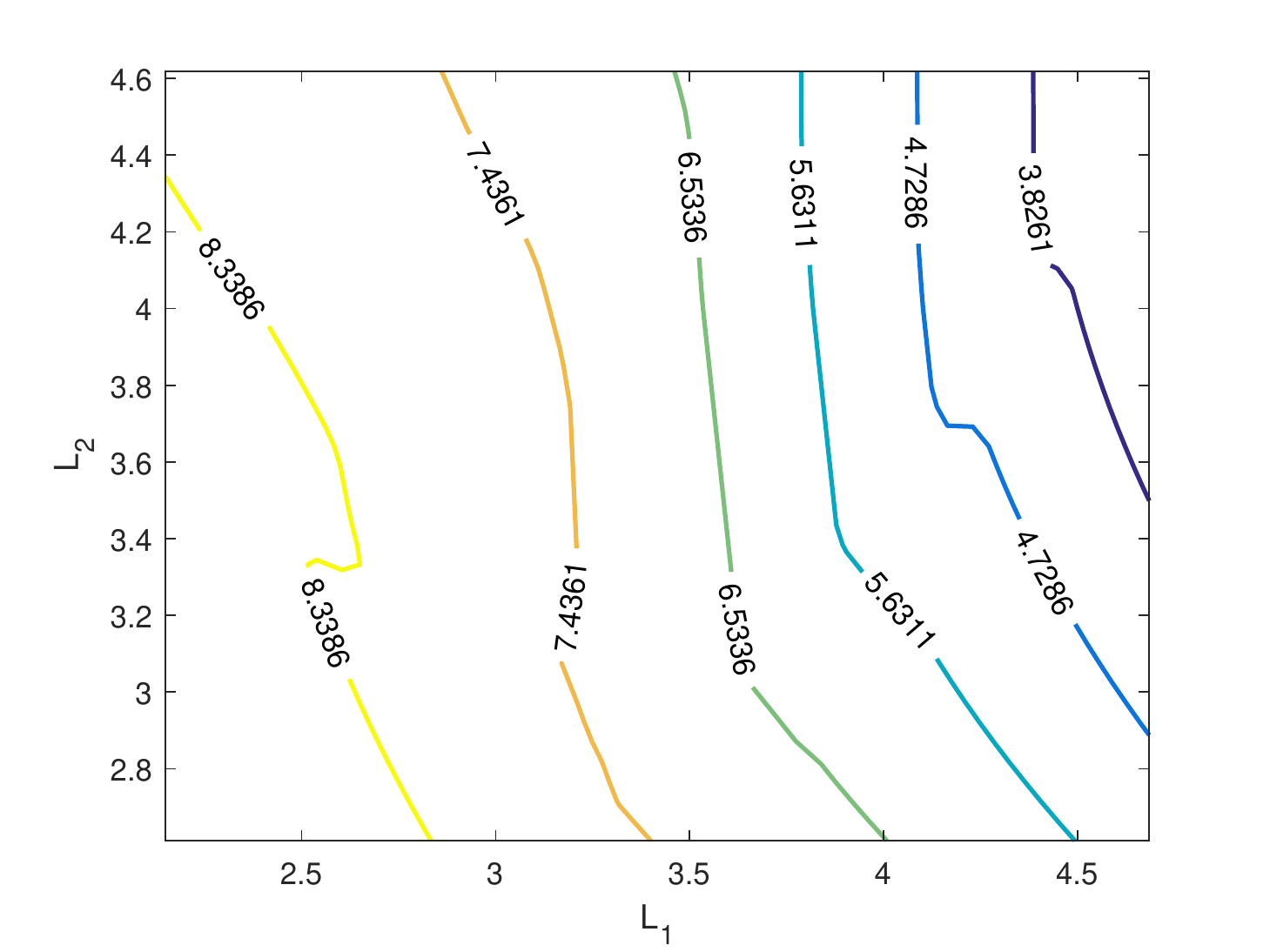}
\caption{The relationship between the configuration robustness and the two parameters $L_1$ and $L_2$ for a robot with two links.}
\label{linkvsrobust}
\end{figure}

Then a parameter synthesis problem ($C.Synthesis$ in Algorithm 1) is solved and we are able to find a solution as shown by the red trajectory as shown in Fig. \ref{fig:b}. The corresponding control policy is shown in Fig. \ref{fig:b}. The parameters are found by using the GP-ACB algorithm to optimize the configuration robustness over the two parameters, $L_1$ and $L_2$, the lengths of the two links (their relationship is shown in Fig. \ref{linkvsrobust}.). They are $\theta_1^* = L_1^* = 2.23$ and $\theta_2^* = L_2^*= 3.35$, respectively.

\section{Conclusion}

The paper presents a new way of modeling and designing reconfigurable robots. We propose a series of concepts, including structural reconfigurable grammar,  structural reconfigurable automaton, and functional reconfigurable grammar, to formally characterize how a reconfigurable robot can be configured and re-configured not only structurally but also functionally. Furthermore, we propose a correct-by-construction design strategy of utilizing such models. We demonstrate with a planar multi-link manipulator and a pick-and-place task as an example to show how such a strategy works.

\bibliography{references}
\bibliographystyle{IEEEtran}

\section*{Appendix}
\textit{THEOREM \ref{node} Given two structural configurations $\omega_1$ and $\omega_2$ with $\omega_1=\omega_2 z$, i.e., $\omega_2$ is a prefix of $\omega_1$, then the following relationship holds:
\begin{equation*}
\rho(\omega_1, \theta_{\omega_1}^*, \mathcal{W}, \bar{x}, \phi) \geq \rho(\omega_2, \theta_{\omega_2}^*, \mathcal{W}, \bar{x}, \phi) 
\end{equation*}
where $\theta_{\omega_1}^*$ and $\theta_{\omega_2}^*$ are the optimal parameters for the two structural configurations $\omega_1$ and $\omega_2$, respectively, in term of configuration robustness. }

\begin{proof}
Since $\omega_1=\omega_2 z$, the parameter space $\Theta_{\omega_2}$ for $\omega_2$ is a subset of the parameters space $\Theta_{\omega_1}$ for $\omega_1$, namely $\Theta_{\omega_2} \subseteq \theta_{\omega_1}$. This implies that given any parameter $\theta_{\omega_1}^* \in \Theta_{\omega_1}$, there exists a parameter $\theta_{\omega_2} \in \Theta_{\omega_2}$ such that $\theta_{\omega_2}=[\theta_{\omega_1}^*,0]$, i.e., $\rho(\omega_1, \theta_{\omega_1}^*, \mathcal{W}, \bar{x}, \phi) = \rho(\omega_2, \theta_{\omega_2}, \mathcal{W}, \bar{x}, \phi)$. Furthermore, given the definition of $\theta_{\omega_2}^*$, we have $\rho(\omega_2, \theta_{\omega_2}^*, \mathcal{W}, \bar{x}, \phi) \geq \rho(\omega_2, \theta_{\omega_2}, \mathcal{W}, \bar{x}, \phi)$. Therefore the theorem has been proved.
\end{proof}

\textit{THEOREM \ref{robustness} Given a functional reconfiguration grammar $FRG=(N,Z,P,F,I)$, a workspace description $\mathcal{W}=(P,\mathcal{H},N)$, and an initial configuration $q_0$ of the FRG, there is a solution to Problem (\ref{problem_1}) if and only if there exists an $\omega$, a $\theta_{\omega}$, and a trajectory $\bar{x}$ generated by the robot built in accordance with $\omega$ and $\theta_{\omega}$, such that  
\begin{equation*}
\rho(\omega, \theta_{\omega}, \mathcal{W}, \bar{x}, \phi) \geq 0
\end{equation*}}

\begin{proof}
Here we first introduce a function $\texttt{ZEROPREFIX}_{\epsilon}:\mathbb{R}_{\geq 0}^{k}\rightarrow \mathbb{N}$ which is defined as follows \cite{shoukryscalable}:
\begin{equation}
\texttt{ZEROPREFIX}_{\epsilon}(g_0,g_1,\cdots,g_{k-1})=\min\quad m \quad s.t.\sum\limits_{j=0}^{m} g_{j} > \epsilon.
\end{equation}
where $\epsilon\in \mathbb{R}_{>0}$ is a constant value. It is quite obvious that the function $\texttt{ZEROPREFIX}_{\epsilon}$ returns the number of zero elements at the beginning of a sequences $g=g_0,g_1,\cdots,g_{k-1}$. 

Next we will present two problems. The first one is related to the function $\texttt{ZEROPREFIX}_{\epsilon}$: 
\begin{prob}
\label{maxzero}
\begin{equation*}
\begin{array}{llll}
\max\limits_{
\begin{array}{ll}
u_0,\cdots,u_{k-1}\in \mathbb{R}^m\\
v_0,\cdots,v_{k-1}\in \mathbb{R}^m\\
g_{0}^{u},\cdots,g_{k-1}^{u}\in \mathbb{R}\\
g_{0}^{v},\cdots,g_{k-1}^{v}\in \mathbb{R}\\
x_{1},\cdots,x_{k}\in\mathbb{R}^{n}
\end{array}}
\texttt{ZEROPREFIX}_{\epsilon}((g_{0}^{u}+g_{0}^{v}),\cdots,(g_{k-1}^{u}+g_{k-1}^{v}))\\
\textit{subject to :}\\
(CA.1) \quad \mathcal{H}(\bar{x})\vDash \phi\\
(CA.2) \quad x_{i+1}=A(x_i)x_{i}+B(x_{i})u_{i}+B^{'}v_{i},\quad i=0,\cdots,k-1\\
(CA.3) \quad \parallel u_{i}\parallel \leq \overline{u}+g_{i}^{u},\quad i=0,\cdots,k-1\\
(CA.4) \quad \parallel v_{i}\parallel \leq g_{i}^{v},\quad i=0,\cdots,k-1\\
(CA.5) \quad 0 \leq g_{i}^{u},0 \leq g_{i}^{v},\quad i=0,\cdots,k-1\\
\end{array}
\end{equation*}
\end{prob}
where $g_{i}^{u}, g_{i}^{v}, i = 0,...,k-1$ are slack variables as those defined in Definition \ref{configurationrobustness}. 

In paper \cite{shoukryscalable}, the authors have shown that solving the above problem can be converted to solving the following problem:
\begin{prob}
\label{minslack}
\begin{equation*}
\begin{array}{llll}
f_{k}=\qquad\min\limits_{
\begin{array}{ll}
u_0,\cdots,u_{k-1}\in \mathbb{R}^m\\
v_0,\cdots,v_{k-1}\in \mathbb{R}^m\\
g_{0}^{u},\cdots,g_{k-1}^{u}\in \mathbb{R}\\
g_{0}^{v},\cdots,g_{k-1}^{v}\in \mathbb{R}\\
x_{1},\cdots,x_{k}\in\mathbb{R}^{n}
\end{array}}
\sum\limits_{i=0}^{k-1} g_{i}^{u}+g_{i}^{v}\\
\textit{subject to :}\\
(CB.1) \quad \mathcal{H}(\bar{x})\vDash \phi\\
(CB.2) \quad x_{i+1}=A(x_i)x_{i}+B(x_{i})u_{i}+B^{'}v_{i},\quad i=0,\cdots,k-1\\
(CB.3) \quad \parallel u_{i}\parallel \leq \overline{u}+g_{i}^{u},\quad i=0,\cdots,k-1\\
(CB.4) \quad \parallel v_{i}\parallel \leq g_{i}^{v},\quad i=0,\cdots,k-1\\
(CB.5) \quad 0 \leq g_{i}^{u},0 \leq g_{i}^{v},\quad i=0,\cdots,k-1\\
(CB.6) \quad \epsilon\left( \sum\limits_{l=0}^{j-1}g_{l}^{u}+g_{l}^{v} \right)\leq g_{j}^{u}+g_{j}^{v},\quad i=1,\cdots,k\\
\end{array}
\end{equation*}
\end{prob}
Specifically, \cite{shoukryscalable} proved that Problem \ref{maxzero} and Problem \ref{minslack} are equivalent, namely, any solution to Problem \ref{maxzero} is also a solution of Problem \ref{minslack}. Based on this conclusion, we can get the following lemma.

\begin{lem}
\label{robustnesszero}
If $\rho(\omega, \theta_{\omega}, \mathcal{W}, \bar{x}, \phi) \geq 0$ with and the corresponding control sequence as $u=u_{0}u_{1}\cdots u_{k-1}$, i.e., applying the $u$ to the robot built in accordance $(\omega, \theta_{\omega}$, then the control sequence $u$ is a optimal solution to Problem \ref{minslack} and $g_{i}^{u}+g_{i}^{v}=0, \forall i=0,1,\cdots,k-1$.
\end{lem}
\begin{proof}
If the configuration robustness satisfies $\rho(\omega, \theta_{\omega}, \mathcal{W}, \bar{x}, \phi) \geq 0$, then according to the definition of configuration robustness, Eqn. (\ref{dfn:robustness}), we have $s_i^u+s_i^v \leq 0$ for all $i=0,...,k-1$. Since the constraints in Definition \ref{configurationrobustness} is stronger than the constraints in Problem \ref{minslack} (the only difference between the two sets of constraints is between $u_i^v$ and $g_i^v$, the former can be negative in Definition \ref{configurationrobustness} while the latter must be non-negative in Problem \ref{minslack}), i.e., any control sequence $u=u_{0}u_{1}\cdots u_{k-1}$ satisfying constraints in Definition \ref{configurationrobustness} will definitely satisfies the constraints in Problem \ref{minslack}. Thus a control sequence $u$ meeting the constraints of Problem \ref{minslack} must satisfy $g_{i}^{u}+g_{i}^{v}=0$ for all $i=0,...,k-1$, which leads subsequently to $f_{k}=0$. From the definition of $f_{k}$ in Problem \ref{minslack}, we have $f_{k} \geq 0$. Thus $f_{k}=0$ is the optimal solution to Problem \ref{minslack} and we have proved the lemma.
\end{proof}

With Lemma \ref{robustnesszero}, we are ready to prove Theorem \ref{robustness}\\

\textbf{Necessity:} Checking whether there exists a control sequence $u=u_{0}u_{1}\cdots u_{k-1}$, which will generate a feasible trajectory $\bar{x}$ to solve Problem \ref{problem_1} can be converted to checking whether there exists a solution to the following problem: 
\begin{prob}
\label{feasible}
\begin{equation*}
\begin{array}{llll}
\min\limits_{
\begin{array}{ll}
u_0,\cdots,u_{i-1}\in \mathbb{R}^m\\
x_{1},\cdots,x_{i}\in\mathbb{R}^{n}
\end{array}}
\quad 1\\
\textit{subjec to :}\\
(CC.1)\qquad\quad\mathcal{H}(\bar{x})\vDash \phi\\
(CC.2)\qquad\quad x_{i+1}=A(x_i)x_{i}+B(x_i)u_{i},\quad i=0,\cdots,k-1\\
(CC.3)\qquad\quad \parallel u_{i}\parallel \leq \overline{u},\quad i=0,\cdots,k-1\\
\end{array}
\end{equation*}
\end{prob}
When there is a feasible trajectory $\bar{x}$, then $g_{i}^{u}+g_{i}^{v}=0, \forall i=0,1,\cdots,k-1$ can be obtained in Problem \ref{minslack}; subsequently $s_{i}^{u}+s_{i}^{v}=0,  \forall i=0,1,\cdots,k-1$ can meet all the constraints in the definition of configuration robustness, which makes $\rho(\omega, \theta_{\omega}, \mathcal{W}, \bar{x}, \phi) = 0$. Therefore $\rho(\omega, \theta_{\omega}, \mathcal{W}, \bar{x}, \phi)  \geq 0$.

\textbf{Sufficiency:} From Lemma \ref{robustnesszero}, if we have $\rho(\omega, \theta_{\omega}, \mathcal{W}, \bar{x}, \phi) \geq 0$, then 
$g_{i}^{u}+g_{i}^{v}=0, \forall i=0,1,\cdots,k-1$ is a solution to Problem \ref{minslack}; subsequently, there is a feasible solution Problem \ref{feasible}, which is also a solution to Problem \ref{problem_1}.
\end{proof}

\end{document}